\documentclass[11pt, letterpaper, USenglish]{article}

\usepackage[a4paper,margin=2.5cm]{geometry}
\usepackage[utf8]{inputenc}
\usepackage{booktabs}
\usepackage[small]{caption}
\usepackage{lmodern}

\usepackage{graphicx}
\usepackage{tikz}
\usepackage{algorithm}
\usepackage{algpseudocode}
\usepackage{xspace}
\usepackage{mathtools}
\usepackage[super]{nth}
\usepackage{amsmath}

\usepackage{amsthm}
\usepackage{amsfonts}
\usepackage{thm-restate}
\usepackage{microtype}
\usepackage{hyperref}
\usepackage{enumitem}
\usepackage{multirow}
\DeclareCaptionFont{xbf}{\bfseries\boldmath}
\usepackage{tablefootnote}

\usepackage{svg}

\usepackage[font=scriptsize]{caption}
\usepackage{comment}

\usepackage{makecell}
\usepackage{pifont}
\usepackage{colortbl}

\usepackage[sort&compress,numbers]{natbib}
\usepackage{cleveref}
\algdef{SE}[SUBALG]{Indent}{EndIndent}{}{\algorithmicend\ }%
\algtext*{Indent}
\algtext*{EndIndent}

\algdef{SE}{Upon}{EndUpon}[1]{\textbf{upon} \(\mbox{#1}\) \textbf{do}}{\textbf{end upon}}%

\theoremstyle{plain}
\newtheorem{theorem}{Theorem}[section]
\newtheorem{definition}[theorem]{Definition}
\newtheorem{lemma}[theorem]{Lemma}

\theoremstyle{remark}

\newenvironment{sloppypar*}
 {\sloppy\ignorespaces}
 {\par}

\newcommand{\qedClaim}{\hfill \ensuremath{\Box}}

\usepackage{subcaption}

\newcommand{\midpointFunction}{\ensuremath{\mathrm{mid}}\xspace}
\newcommand{\longestEdge}{\ensuremath{\mathrm{E}_{\mathrm{max}}}\xspace}
\newcommand{\geometricMedian}{\ensuremath{\mathrm{Geo}}\xspace}
\newcommand{\geoSet}{\ensuremath{S_{\mathrm{geo}}}\xspace}
\newcommand{\trueBox}{\ensuremath{\mathrm{TH}}\xspace}
\newcommand{\medianBox}{\ensuremath{\mathrm{GH}}\xspace}
\newcommand{\radEnc}{\ensuremath{\mathrm{r}_{\mathrm{cov}}}\xspace}
\newcommand{\convexHull}{\ensuremath{\mathrm{Conv}}\xspace}
\newcommand{\ball}{\ensuremath{\mathcal{B}}\xspace}
\newcommand{\geoMD}{\ensuremath{\mathrm{MD}_{\mathrm{geo}}}\xspace}
\newcommand{\MD}{\ensuremath{\mathrm{MD}}\xspace}
\newcommand{\boxmean}{\ensuremath{\mathrm{BOX\!-\!MEAN}}\xspace}
\newcommand{\boxmedian}{\ensuremath{\mathrm{BOX\!-\!GEOM}}\xspace}
\newcommand{\mdmean}{\ensuremath{\mathrm{MD\!-\!MEAN}}\xspace}
\newcommand{\mdmedian}{\ensuremath{\mathrm{MD\!-\!GEOM}}\xspace}

\newcommand{\msg}{\texttt{msg}}
\newcommand{\trueCent}{\ensuremath{\nu^*}\xspace}
\newcommand{\trueMedian}{\ensuremath{\mu^*}\xspace}
\newcommand{\distance}{\ensuremath{\mathrm{dist}}\xspace}

\DeclareMathOperator*{\argmin}{arg\,min}

\begin{document}

\title{Approximate Agreement Algorithms for Byzantine Collaborative Learning}

\author{M\'elanie Cambus\footnote{Aalto university, Finland, melanie.cambus@aalto.fi} 
\and Darya Melnyk\footnote{TU Berlin, Germany, \{melnyk@tu-berlin.de, tijana.milentijevic@campus.tu-berlin.de, schmiste@gmail.com\}}
\and Tijana Milentijevi\'c\footnotemark[2]
\and Stefan Schmid\footnotemark[2] 
}

\date{}

\maketitle

\begin{abstract}
  In Byzantine collaborative learning, $n$ clients in a peer-to-peer network collectively learn a model without sharing their data by exchanging and aggregating stochastic gradient estimates. Byzantine clients can prevent others from collecting identical sets of gradient estimates. The aggregation step thus needs to be combined with an efficient (approximate) agreement subroutine to ensure convergence of the training process.
  In this work, we study the geometric median aggregation rule for Byzantine collaborative learning. We show that known approaches do not provide theoretical guarantees on convergence or gradient quality in the agreement subroutine. To satisfy these theoretical guarantees, we present a hyperbox algorithm for geometric median aggregation.  
  We practically evaluate our algorithm in both centralized and decentralized settings under Byzantine attacks on non-i.i.d. data. We show that our geometric median-based approaches can tolerate sign-flip attacks better than known mean-based approaches from the literature. 
\end{abstract}


\section{Introduction}

Distributed machine learning is an attractive alternative to traditional centralized training. By distributing the training process across multiple peers, computations can be parallelized and scaled up, while peers can retain their individual datasets locally: peers simply need to exchange and aggregate their stochastic gradient estimates. Distributed machine learning hence also provides peers with a degree of autonomy \cite{byzantine-primer}, which, combined with additional cryptographic techniques, improves privacy \cite{Grivet_S_bert_2021, zhang2020batchcrypt, differential-privacy-fed-learning, noble2023differentiallyprivatefederatedlearning}.

However, distributed machine learning also introduces new challenges. In particular, to ensure high-quality models as well as convergence of the training process, gradient aggregation requires that the peers agree on a similar set of vectors. Exact distributed agreement algorithms where the peers agree on the same vector are costly. We therefore allow the peers to compute output vectors that are close to each other, but not necessarily identical. This agreement type is called approximate agreement. Achieving approximate agreement is particularly challenging in distributed settings where some peers may be faulty or even malicious. Additionally, large-scale machine learning systems rely on user-generated data, which can be maliciously manipulated. 

This paper studies the approximate agreement problem in distributed training where some peers may be Byzantine. We focus on parameter-based attacks which involve altering local parameters, such as gradient or model weights \cite{shi2022challenges}. Parameter modification can be done randomly and non-randomly. Non-random modification includes altering the direction or size of the parameters based on the model learned from the local dataset. Possible non-random modification attacks are flipping the sign of gradients or increasing the magnitudes. Random modification implies randomly sampling a number and treating it as one of the parameters of the local model.

We present and revisit several gradient aggregation algorithms and study their robustness. As mean-based aggregation is sensitive to outliers, we are particularly interested in (geometric) median-based aggregation. We analyze the theoretical guarantees of different algorithms with respect to their approximation guarantee (how close they get to the geometric median of non-faulty peers).  
We also show that the prevalent safe area approaches for solving multidimensional approximate agreement do not give satisfying guarantees with respect to the geometric median. 
We complement our theoretical considerations with an empirical evaluation, studying the performance of different algorithms under crash failures and sign attacks.
For comparison, we also implement the MDA algorithm by El-Mhamdi et al.~\cite{jungle}, and the recently introduced box algorithm by Cambus et al.~\cite{centroid-paper} which uses the mean instead of the geometric median. 

\subsection{Our contributions}

In this section, we summarize our contributions. 
Note that Contributions (\ref{contrib:item1}) and (\ref{contrib:item2}) focus on algorithms used by clients to agree on an aggregation vector during one single learning round. Contributions (\ref{contrib:item3}) and (\ref{contrib:item4}) then empirically study the behavior and implications of our algorithm when executed in multiple rounds.

\begin{enumerate}
    \item\label{contrib:item1} We study gradient aggregation via the geometric median. Our goal is to approximate this popular aggregation vector in the distributed setting. To this end, we adapt popular agreement algorithms to the context of geometric median approximation. 
    The approximation of the geometric median is defined analogously to the approximation of the mean in~\cite{centroid-paper}. 
    We analytically show that agreement in the safe area (often considered in the literature to solve multidimensional approximate agreement) does not compute a vector that provides a bounded approximation of the geometric median. 
    We further prove that also the medoid-based aggregation rule Krum~\cite{blanchard2017machine,byzantine-primer} does not provide a bounded approximation of the geometric median. 
    Regarding the natural approach of approximate agreement based on computing the minimum diameter~\cite{jungle} and then applying the median aggregation rule, we formally show that this solution may not even converge.
    
    \item\label{contrib:item2} The results in~(\ref{contrib:item1}) show that existing algorithms do not provide bounded approximations of the geometric median. We present an algorithm based on hyperboxes that achieves a $2\sqrt{d}$-approximation of the geometric median and converges, where $d$ is the dimension of input vectors and $n$ is the number of nodes. We formally prove the respective properties for approximate agreement. 
    
    \item\label{contrib:item3} We empirically evaluate our algorithm for centralized and distributed collaborative learning. To this end, we consider non-i.i.d. data split among $10$ clients, one of whom is Byzantine. We study the algorithm under various Byzantine behaviors, such as crash failures and reversed gradient. Our results show that an accuracy of over $78\%$ can be achieved in all settings when using the hyperbox algorithm for the geometric median aggregation rule. 
    
    \item\label{contrib:item4} We empirically compare our results to known averaging agreement algorithms from the literature, such as minimum-diameter averaging \cite{jungle}, box algorithm for the mean \cite{centroid-paper}, simple geometric median and simple mean aggregation rules in distributed collaborative learning. In centralized collaborative learning, we additionally consider the Krum and Multi-Krum aggregation rules~\cite{blanchard2017machine,byzantine-primer}. We also provide a first practical evaluation of the box algorithm for the mean. 

\end{enumerate}
As a contribution to the research community, to facilitate follow-up work and ensure reproducibility, we will share our source code and experimental artifacts together with this paper (once accepted).

\subsection{Organization}

The remainder of this paper is organized as follows. We start by introducing collaborative learning and approximate agreement in Section~\ref{sec: background}. In Section~\ref{sec: model}, we present the theoretical definition of the approximation of the geometric median in the Byzantine setting and other definitions needed for the studied algorithms. Section~\ref{sec: theory} shows that some known strategies to approximate the geometric median fail and presents the hyperbox approach that provides a $2\sqrt{d}$-approximation. In Section~\ref{sec: practice}, we present our practical evaluation of the algorithms. Finally, in Section~\ref{sec: related_work}, we summarize the related work and conclude with a summary of this paper in Section~\ref{sec: conclusion}.

\section{Preliminaries}\label{sec: background}

We introduce here the concepts necessary for our contributions, first presenting the machine learning context and then giving the theoretical background regarding Byzantine agreement.

\subsection{Distributed machine learning}

We consider a system with $n$ nodes, also called clients, that have input vectors $v_1, \dots, v_n \in \mathbb{R}^d$. 
Each client $u_i$ has access to its own data, which follows an unknown distribution $\mathcal{D}_i$. 
In this system, we allow certain clients to be faulty and crash or send corrupted vectors. 
Non-faulty clients try to learn the parameters $\theta^i$ of a machine learning model, that ensures optimal accuracy over the joint data distribution across all clients in the system \cite{byzantine-primer}. Specifically, for a given parameter vector $\theta$, also known as weight vectors, and a data point $v \in \mathbb{R}^d$, the model incurs a real-valued loss $q(\theta , v)$. This function calculates how well the model with parameters $\theta$ predicts a given data point $v$. Therefore, each client's local loss function is:
\begin{equation} \label{eq:gradient}
    Q_i(\theta) = \mathbb{E}_{v\sim \mathcal{D}_i} [q(\theta, v)] \qquad \forall \theta \in \mathbb{R}^d.
\end{equation}
We make following assumptions on the loss functions of non-faulty nodes, denoted by $h$:
\begin{enumerate}
    \item Loss function $q$ must be differentiable with respect to $\theta$.
    \item Local loss functions $Q_i$ are non-negative, i.e. $Q_i \geq 0$ for all non-faulty nodes $u_i, \forall i \in [h]$.
    \item Local loss functions are $L$-smooth \cite{bottou2018optimization}, i.e. there exists a constant $L$ such that $\forall \theta, \theta' \in \mathbb{R}^d, \forall j \in [h]$:
    \begin{equation*}
         \Big\| \nabla Q_j(\theta) - \nabla Q_j(\theta') \Big\|_2 \leq L \Big\| \theta - \theta' \Big\|_2
    \end{equation*}
\end{enumerate}

Clients must work together in the system to solve the optimization problem, due to the differences in the data distributions \cite{boyd2011distributed}. However, when data distributions are identical, collaboration remains beneficial as it reduces the computational costs on each client, making the learning process more efficient \cite{boyd2004convex, decentralized-outperform}.

\textbf{Centralized collaborative learning model. }
In the centralized collaborative learning model, there is one server that coordinates the learning process. The dataset is split among clients and is preserved locally. Each client has a local model and at the beginning of every round, the weights of the local model are set to the weights of the global model. Clients compute a stochastic estimate $g^{(i)}_t$ of the gradient $\nabla Q_i(\theta^{(i)}_t)$ for all local models $\theta^{(i)}_t$ in iteration $t$.
The gradient estimate $g^{(i)}_t$ is computed by drawing a data point $v$ or a sample from the local data generating distribution $\mathcal{D}_i$:
\begin{equation} \label{eq:gradient-centralized}
    g^{(i)}_t = \nabla q(\theta^{(i)}_t, v) \qquad \text{with} \quad v\sim \mathcal{D}_i.
\end{equation}
With the help of Equation~\ref{eq:gradient}, the gradient estimate $g^{(i)}_t$ equals the true gradient $\nabla Q_i(\theta^{(i)}_t)$ in expectation. 

The global entity then receives stochastic gradients $g_t$ from all clients and computes an aggregate of the received messages $\widehat{g_t}$. Consequently, the global model's parameter $\theta_t$ is updated to $\theta_{t+1}$ as follows:
\begin{equation*}
    \theta_{t+1} = \theta_t - \gamma_t \cdot \widehat{g_t}
\end{equation*}
with $\gamma_t$ being the learning rate. In the next round, local models again set their weights to the weights of the global model and the procedure is repeated. The number of iterations $T$ is parameterized and decided on in advance, before the learning process begins. The algorithm stops after $T$ iterations. The performance of the global model is measured after every round and the accuracy is reported.

\textbf{Decentralized collaborative learning model. }
The problems in centralized collaborative learning occurs when transferring a large machine learning model from a central server. First, the communication cost is high, since the learning process is done in $T$ iterations and the parameters of the central model are sent to all clients at each round. Second, the central server decides on the global update, which does not necessarily suit all clients since they do not follow the same data distribution. Finally, the central server is also a single point of failure.

A natural way to address these drawbacks is to decentralize the model. In this architecture, there is no global entity. As in the centralized model, the data is split among clients and is kept locally. Each client has a local model which is created once at the beginning and stored for updating throughout the iterations. 
Each client $u_i$ computes a stochastic gradient $g^{(i)}_t$ of its local loss function's gradient $\nabla Q_i(\theta^{(i)}_t)$ in the same way as in Equation~\ref{eq:gradient-centralized} in the centralized collaborative learning model. 
However in this decentralized model, clients broadcast their gradients $g^{(i)}_t$ to all other clients in the system. Each client then gathers gradients from all other clients, and aggregates them using an aggregation function.

It cannot be guaranteed that clients agree on the same gradient aggregation, as there is no central server maintaining a global model and faults can occur during the communication process. 
To ensure gradient aggregations to be as close as possible in between clients, we use agreement algorithms that run in multiple sub-rounds. At each sub-round, each client sends its vector to all other clients. Upon receiving the messages, each client performs an aggregation rule to these vectors. This output will be the input of the next sub-round. The number of sub-rounds is predefined and in this work we choose $\log t$ sub-rounds, where $t$ is the "big" iteration. This result is taken from the El-Mhamdi et al. work \cite{jungle}.
In the last sub-round of iteration $t$, clients update their models and enter the next iteration $t+1$. The process is repeated until the stopping criteria is met. Local models are tested after each iteration and their accuracy is reported.

\subsection{Aggregation rules}
This work compares different ways of computing aggregates of the clients' local gradients after each round of the learning process, in both the centralized and the decentralized collaborative learning models. We aim at assessing how well those aggregation rules react to the presence of faulty clients in the system. More precisely, the aggregation rules we consider are the geometric median and the mean of correct input vectors. 

The \textbf{mean} is defined as follows:
\begin{definition}[Mean]
    The mean of a finite set of $n$ vectors ${v_i, i\in [n]}$ is $$\frac{1}{n}\sum_{i=1}^n v_i.$$
\end{definition}
We denote by $\trueMedian$ the true geometric median and $\trueCent$ the true mean, which is computed only from vectors of non-faulty nodes. 

The geometric median minimizes the sum of Euclidean distances to all points in the system.
We define the \textbf{geometric median}, following the definition provided by Small~\cite{geom-definition}.
\begin{definition}[Geometric median]
    Consider a set of $n$ vectors $\{v_{1},v_{2},\dots ,v_{n}\}$ with each $v_{i}\in \mathbb {R} ^{d}$, the geometric median of this set, denoted $\geometricMedian(\{v_{1},v_{2},\dots ,v_{n}\})$, is defined as
    $${\underset {\mu\in \mathbb {R} ^{d}}{\operatorname {arg\,min} }}\sum _{i=1}^{n}\left\|v_{i}-\mu \right\|_{2}.$$
\end{definition}

Also other aggregation rules have been considered in the literature. We will compare our practical results to two popular aggregation rules -- Krum and Multi-Krum~\cite{blanchard2017machine,byzantine-primer}. 
The \textbf{Krum} aggregation rule is based on computing the medoids on subsets of $n-t$ vectors and choosing the medoid with the smallest total distance. Let $\{v_1,\ldots,v_k\}, k\ge n-t$ be the set of vectors received by the server. Let $C_j$ denote the set containing the indices of the $n-t-1$ closest vectors to $v_j$ from the set $\{v_1,\ldots,v_k\}\setminus v_j$. Then, 
\begin{align}\label{eq:Krum}
    \text{Krum}(v_1,\ldots, v_k) = v_i, \ \text{where} \quad i=\argmin_{i\in[n]} \sum_{l\in C_i} \lVert v_i - v_l\rVert.
\end{align}

\textbf{Multi-Krum}\label{krum}  is a generalization of Krum, where, instead of selecting one vector minimizing the sum of distances, the average of $q$ such vectors with smallest distances is chosen. Let $M(q)$ denote the set that contains $q$ vectors with smallest total distances to their $n-t-1$ closest neighbors. Then,
\begin{align}\label{eq:multi-krum}
    \text{Multi-Krum}_q(v_1,\ldots, v_k) = \frac{1}{q}\sum_{i\in M(q)} v_i.
\end{align}

\subsection{Multidimensional approximate agreement}\label{subsec: multidimApproxAgreement}

To be able to aggregate the local gradients in the presence of faulty nodes, we need algorithms that take into account potential faults. We hence study algorithms that allow nodes to agree on a vector in the presence of faulty nodes in the system. This problem is referred to as multidimensional approximate agreement.

We assume that $n$ nodes communicate with each other in a peer-to-peer fashion to agree on a common output. The communication is assumed reliable~\cite{BRACHA1987130, TouegRB}: let some node $u$ reliably broadcast a message $\msg_u$ and let $\msg_u(u_i)$ and $\msg_u(u_j)$ be the message from node $u$ reliably received by nodes $u_i$ and $u_j$ respectively, then $\msg_u(u_i)=\msg_u(u_j)$. 
In the \emph{Byzantine agreement} problem, the task is to agree on a common output in the presence of $t\le n/3$ arbitrary node failures, known as \emph{Byzantine failures}. 
Motivated by the machine learning application, we consider multidimensional inputs. More specifically, the input of each node is a vector in $\mathbb{R}^d$, where $d$ is the dimension of the vector. 
We assume that the communication between nodes is synchronized, i.e., the nodes are communicating in rounds. 
Synchronous Byzantine agreement requires $t+1$ rounds~\cite{FischerLynchMinRounds}, which is slow if many Byzantine nodes can be present in the system. Hence, similarly to~\cite{jungle}, we relax the agreement condition. We consider $\varepsilon$-approximate agreement, which only requires the output vectors $v_i$ and $v_j$ of any two nodes $u_i$ and $u_j$ to satisfy $\lVert v_i - v_j\rVert_2 < \varepsilon$. 

The standard algorithm for multidimensional approximate agreement, referred to later in this paper as the \textbf{safe area algorithm}, is based on each node repeatedly computing a vector inside a polytope called the safe area, and sharing it with the other nodes in the next round. Formally, the safe area is defined as follows.

\begin{definition}[Safe area~\cite{multidim-approx-agreement}]
    Consider $n$ vectors $\{v_1, \dots ,v_n\}\eqqcolon V$, where $t<n/(\max\{3,d+1\})$ of which can be Byzantine. Let $C_1,\ldots, C_{\binom{n}{n-t}}$ be the convex hulls of every subset of $V$ of size $n-t$. The \textit{safe area} is the intersection of these convex hulls: $$\bigcap_{i\in \left[\binom{n}{n-t}\right]} C_i.$$
\end{definition}

In~\cite{multidim-approx-agreement}, the authors show that the safe area exists (i.e., the intersection is non-empty) if $t<n/(\max\{3,d+1\})$. The strong guarantee this algorithm gives on its output makes it popular. Indeed, the output vector of each node is guaranteed to be in the convex hull of all non-faulty input vectors. However, the condition $t<n/(\max\{3,d+1\})$ implies that the algorithm cannot be used in the presence of faulty nodes when $n\leq d$, which is the case in our distributed machine learning setting. The safe area algorithm is hence only of theoretical interest to us. 

Another algorithm aiming at solving multidimensional approximate agreement is the \textbf{Minimum Diameter Averaging (MDA) algorithm}~\cite{jungle}. This algorithm works as follows. In each round, the nodes receive the messages and determine a set $\MD$ of $n-t$ received vectors that has the minimum diameter (note that such a set is not unique). The new input vector for the following round is computed as the mean of all vectors in $\MD$. 
Observe that the output vector is not necessarily inside the convex hull of all non-faulty input vectors.

Recently, another algorithm has been introduced to approximate the mean aggregation rule~\cite{centroid-paper}. This algorithm, referred to in this work as the \textbf{hyperbox algorithm}, is based on picking a vector in the intersection of hyperboxes. The computed output vector of each node is guaranteed to be inside a so-called trusted hyperbox, which is defined as follows. 

\begin{definition}[Trusted hyperbox]
    Let $f\leq t$ be the number of Byzantine nodes and let $v^*_i$, $i \in [n-f]$ denote the true vectors. Let $v^*_i[k]$ denote the $k^{th}$ coordinates of these vectors. 
    The \emph{trusted hyperbox} $\trueBox$ is the Cartesian product of  
    $\trueBox[k]\coloneqq\bigl[\min_{i\in [n-f]}v^*_i[k], \max_{i\in [n-f]}v^*_i[k]\bigr], $ for all $k\in[d].$

\end{definition}

Since the trusted hyperbox cannot be computed locally, the algorithm is based on computing local hyperboxes that are guaranteed to lie inside $\trueBox$.

\begin{definition}[Locally trusted hyperbox]
    Let $v_1, \dots, v_{m_i}$ be the vectors received by node $u_i$, where $m_i$ is the number of messages received by node $u_i$. 
    The number of Byzantine values for each coordinate is at most $m_i-(n-t)$. 
    Denote $\phi_i: [m_i] \to [m_i]$ a bijection s.t. $v_{\phi_i(j_1)} [k]\leq v_{\phi_i(j_2)} [k], \forall j_1, j_2\in [m_i]$. 
    The locally trusted hyperbox computed by node $u_i$ is the Cartesian product of $\trueBox_i[k] \coloneqq \bigl[v_{\phi_i(m_i-(n-t)+1)}[k], v_{\phi_i(n-t)}[k]\bigr]$ for all $k\in[d]$.

\end{definition}

The algorithm works as follows: Let $m\ge n-t$ be the number of messages received by node $u$ in round $r$. Node $u$ computes the means $A_1,\ldots, A_{\binom{m}{n-t}}$ of every subset of $n-t$ received vectors. Then, the intersection of the locally trusted hyperbox and the smallest coordinate-parallel hyperbox containing $A_1,\ldots, A_{\binom{q}{n-t}}$ is computed. The new input vector for round $r+1$ is computed as the midpoint (see \Cref{def:midpoint}) of the intersection of the hyperboxes.

\section{Model}\label{sec: model}

In order to define different algorithms that aim at getting as close as possible to the geometric median of non-faulty nodes, we need a measure of "how close" the output of an algorithm is from this true geometric median. However, we need to do so considering how close it is possible to get to this true geometric median since achieving a "perfect" output in the presence of Byzantine nodes is impossible. We hence start by defining an approximation ratio for the geometric median, as in~\cite{centroid-paper}. 
We then give some definitions that will be necessary to adapting the MDA and hyperbox algorithms to the geometric median aggregation rule. 

From now on, $t$ refers to the maximum number of Byzantine nodes that can be present in the system, and $f \leq t$ refers to the true amount of Byzantine nodes (not known by non-faulty nodes). We also sometimes refer to non-faulty nodes as \textit{true nodes}, and vectors from non-faulty nodes as \textit{true vectors}. Hence, the geometric median of non-faulty nodes will be called \textit{true geometric median}, and similarly for the mean. 

\subsection{Approximation of the geometric median in the Byzantine setting}

Since in the Byzantine setting there are instances in which no algorithm can identify the faulty nodes, we need to consider the set of all possibly non-faulty geometric medians if $t$ Byzantine nodes were present in the system.

\begin{definition}
    We define $\geoSet$ as 
    \begin{align*}
        \geoSet = \{\geometricMedian(\{v_i, \forall i\in I\}) ~ |~ I\subseteq [n], |I| = n-t\}.
    \end{align*}
\end{definition}

In order to use $\geoSet$ as a basis for the definition of approximation of the true geometric median in the Byzantine setting, we need to prove the following lemma.

\begin{lemma} \label{lemma:geom-in-convex}
    The true geometric median $\trueMedian$ is inside the convex hull of possible geometric medians of each correct node: $\trueMedian \in \convexHull(\geoSet(i))$, for all $i \in [n]$. 
\end{lemma}
\begin{proof}
    Each correct node $u_i$ computes geometric medians on subsets of size $n-t$ of received vectors. If $f=t$, then node $u_i$ would compute the true geometric median among other geometric medians, which implies $\trueMedian \in \geoSet(i)$. 
    If $f<t$, then the true geometric median is not necessarily an element of $\geoSet(i)$. In this case, there are multiple subsets of size $n-t$ containing only true nodes. Let us denote the geometric medians of these subsets by $\mu_1, \mu_2, \dots,\mu_k$ and call them partial true geometric medians. Note that each partial true geometric median differs from other partial true geometric medians by at least one vector from the subset it is computed on. Now, consider the convex hull spanned only by these partial true geometric medians. 
    Each face of this convex hull corresponds to a hyperplane that splits the space into two half-spaces. 

    One of those half spaces (including its boundary) contains all partial true medians. By definition of the geometric median, for each of those partial true medians, at least $\frac{n-t}{2}$ of the vectors it is computed on are in the same half space as the median. 
    This is true for each partial true median and each differs by a vector, meaning that at least $\frac{n-t}{2} + {n-f\choose n-t}-1\geq \frac{n-f}{2}$ vectors are in the half space containing the convex hull. 
    However, the true geometric median has to be in the half space that contains at least $\frac{n-f}{2}$ nodes (it is computed on $n-f$ nodes). Hence, it will be in the half space containing the convex hull. 

    Since this is true for every face of the convex hull, the true geometric median has to be included in the intersection of all half-spaces that have a face as boundary and contain the convex hull, i.e. the true geometric median has to be contained in the convex hull. 
    
    Finally, the convex hull of partial true geometric medians (the one containing the true geometric median as just shown) is included in the convex hull of $\geoSet$ (by inclusion of the sets). 
    Hence, the true geometric median $\trueMedian$ is inside the convex hull of all possible geometric medians $\convexHull(\geoSet(i))$.
\end{proof}

Now, we want to compute the closest possible vector to the true geometric median, which means getting as close as possible to the center of $\convexHull(\geoSet)$. 
Finding this center is equivalent to finding the center of the minimum covering ball around $\geoSet$. 
Hence, the best possible approximation vector of the true geometric median is the center of the minimum covering ball around the set of possible geometric medians $\ball(\geoSet)$.
The approximation definition follows.

\begin{definition}\label{def:approximation} 
    Let $\radEnc$ be the radius of the minimum covering ball of $\geoSet$. All vectors found at distance at most $c \cdot \radEnc$ from the true geometric median $\trueMedian$ provide an $c$-approximation of $\trueMedian$.
\end{definition}

\subsection{Minimum diameter approach for the geometric median}

In order to adapt the minimum diameter averaging approach mentioned in \Cref{subsec: multidimApproxAgreement} to the geometric median aggregation rule, we here formally define a subset of the initial vectors that has minimum diameter. 

\begin{definition}
    Consider a set of vectors $\{v_1, \dots, v_n\}$ s.t. $v_i\in\mathbb{R}^d, \forall i\in[n]$. The set $\geoMD$ is a subset of $\{v_1, \dots, v_n\}$ of size $n-t$ s.t.
    \begin{align*}
        \geoMD\in \arg \min_{\substack{I\subseteq[n]\\ |I|=n-t}}\max_{i, j\in I} \lVert v_i - v_j \rVert_2.
    \end{align*}
\end{definition}

\subsection{Hyperboxes approach for the geometric median}

In order to adapt the hyperbox algorithm described in \Cref{subsec: multidimApproxAgreement} to the geometric median aggregation rule, we here formally define the geometric median hyperbox, the midpoint function, and also the maximum length edge of a hyperbox.

\begin{definition}[Geometric median hyperbox]
    The geometric median hyperbox $\medianBox$ is the smallest hyperbox containing $\geoSet$, and the local geometric median hyperbox $\medianBox_i$ of node $u_i, \forall i \in [n]$ is the smallest hyperbox containing $\geoSet(i)$. 
\end{definition}

\begin{definition}[Midpoint]\label{def:midpoint}
The $\midpointFunction$ of a hyperbox $X$ is defined as 
$$\midpointFunction(X) = \Bigl(\midpointFunction\bigl(X[1]\bigr), \dots, \midpointFunction\bigl(X[d]\bigr)\Bigr), $$
where $X[k]$ is the set containing all $k^{th}$ coordinates of vectors of the set $X$, and the one-dimensional $\midpointFunction$ function returns the midpoint of the interval spanned by a finite multiset of real values.
\end{definition}

\begin{definition}[Maximum length edge (\longestEdge)]
     The length of the edge of maximum length of a hyperbox $H$ is defined as
         $$\longestEdge(H) = \max_{\substack{k\in[d]\\ v, w\in H}} \bigl|v[k] - w[k]\bigr|.$$    
 \end{definition}

\section{Theoretical analysis}\label{sec: theory}

In this section, we analyze algorithms that allow gradient aggregation in a single iteration of the learning process. Hence, when referring to convergence in this section, we talk about the convergence of an algorithm that aggregates several gradients for one single aggregation step. 

\subsection{Safe area, Krum, and $\boldsymbol{\MD}$ approaches}

Even though the standard approach for multidimensional approximate agreement is using a safe area algorithm, and such algorithms give strong guarantees on the output vector of each node as mentioned earlier in \Cref{subsec: multidimApproxAgreement}, it might not be the best way to agree on the true geometric median. 

\begin{theorem} \label{Safearea-unbounded-geom}
The approximation ratio of the true geometric median using the safe area algorithm is unbounded.
\end{theorem}
\begin{proof}
Assume the number of correct nodes is $d\cdot f+1$ and the number of Byzantine nodes is $f$. Let $v_0=(0,0,\dots,0)$ be the input of one of the correct nodes and all Byzantine nodes. The rest of the correct nodes are divided into $d$ groups of $f$ nodes and have input vectors $v_i$ for $i\in \{1,\dots,d\}$. 
These groups of nodes are at most $\epsilon$ far apart. We denote vector $v=(x,0,\dots,0)$ and $\epsilon_j=\epsilon \cdot e_j$, with $e_j$ as $j^{th}$ unit vector. Then, the input vectors of $d$ groups is $v_j=\epsilon_j+v$, for $j\in [d]$.

In order to compute the safe area, we must consider the hyperplanes spanned by all subsets of $(n-f)$ nodes. Note that all hyperplanes are distinct and they intersect only at the point $v_0$. Therefore, the point $v_0$ is the safe area.

There we can differentiate between two extreme medians: the true geometric median and the geometric median containing all Byzantine nodes and $d\cdot f-f$ nodes with vector $v$. The true geometric median contains all correct nodes and coincides with the vector $v$, so $\trueMedian=(x,0,\dots,0)$. Distance from the optimal median to the safe area is $x$. In the second extreme median, we consider $f+1$ nodes with input $v_0$ and $dt-t$ nodes with input $v$. These points lie on a line. Therefore, the geometric median is the same as median. If the dimension $d>3$ or $d=3$ and $f>1$, then the computed geometric median $\mu$ is equal to the true median $\trueMedian = (x,0,\dots,0)$. These two geometric medians define the diameter of the minimum covering ball of all possible geometric medians. The radius of the ball is 0, since the two extreme medians coincide. Therefore, the competitive ratio of the safe area approach is:
\begin{center}
    $\distance(\text{safe area}, \trueMedian) \leq c \cdot \radEnc .$
\end{center}
Constant $c$ must be $c = \infty$, which implies that the approximation ratio is unbounded.

If the dimension $d=3$ and $f=1$, there are in total $5$ nodes, where $3$ non-faulty nodes are located at vector $v$, one non-faulty node is at $v_0$ and one Byzantine is also at $v_0$. We can, similarly to the first case, define two extreme medians. The true geometric median containing only correct nodes is at vector $v$. The other extreme geometric median is defined by two points at vector $v_0$ and two points at vector $v$. The median can be any point in the line segment between these two vectors. W.l.o.g. we choose the midpoint to be the geometric median. The distance between the true geometric median and the safe area is $x$. The diameter of the minimum covering ball defined by two extreme medians is $\frac{x}{2}$. The radius of the ball is $\frac{x}{4}$. The approximation ratio is: 
\begin{align*}
\frac{\distance(\text{safe area}, \trueMedian)}{\radEnc} =
\frac{x}{(x/4)} = 4.    
\end{align*}

\end{proof}

Another algorithm proposed to solve multidimensional approximate agreement referred to in \Cref{subsec: multidimApproxAgreement} is MDA. Here is a proposed adapted version of this algorithm for the geometric median aggregation rule. 

\begin{algorithm}
\caption{Minimum Diameter Approach}\label{alg:md}
\begin{algorithmic}[1]
\For{\textit{each round $r=1,2,\dots$}}
\For{\textit{each node $u_i$ with input $v_i$:}}
    \State \textit{broadcast $v_i$ reliably to all nodes}
    \State \textit{receive up to $n$ messages $M_i= \{v_j, j \in [n]\}$}
    \State \textit{compute} $\MD(M_i,n-t)$
    \State \textit{set} $v_i \leftarrow \geometricMedian\left(\MD(M_i,n-t)\right)$
\EndFor
\EndFor
\end{algorithmic}
\end{algorithm}

The MDA algorithm was shown in~\cite{centroid-paper} to give a good approximation of the mean aggregation rule. However, we show here that, adapted to the geometric median, the algorithm does not always converge. That is, the $\MD$ algorithm for the geometric median aggregation rule does not solve the multidimensional approximate agreement problem.  

\begin{lemma}\label{lem:md_no_convergence}
\Cref{alg:md} does not converge. 
\end{lemma}
\begin{proof}
    Consider the following setting: $n-t$ nodes in the system are non-faulty, $(n-t)/2$ of those starting with vector $v_1$ and the others starting with vector $v_2$. 
    Suppose now that Byzantine nodes pick vectors $v_1$ and $v_2$ (half of them on each vector), then the minimum covering ball of $v_1$ and $v_2$ is the minimum diameter ball that will be considered by all nodes during the first round of the algorithm. 
    Denote $D\coloneqq \lVert v_1-v_2 \rVert_2$ the diameter of this ball. 
    
    Moreover, let us consider the case where Byzantine nodes that chose $v_1$ only send their vector to half of the true nodes (denote $U_1$ this set), and Byzantine vectors who chose vector $v_2$ send their vector to the other true nodes (denote $U_2$ this set).

    Nodes in set $U_1$ receive $n-t + t/2$ vectors, and will choose a set of $n-t$ vectors of diameter $D$. However, all such sets have diameter $D$, and only one single set has $(n-t)/2$ vectors on $v_1$ and $v_2$ respectively. Hence, all sets of $n-t$ vectors of diameter $D$ but one have $v_1$ as a median. Thus, it is possible that all vectors in $U_1$ pick $v_1$ as their vector for starting round $2$.

    Similarly for nodes in set $U_2$, it is possible that all vectors in $U_2$ pick $v_2$ as their vector for starting round $2$.

    We then find ourselves at the beginning of round $2$ in the exact same configuration as in the beginning of round $1$. The Byzantine nodes hence just need to repeat this behavior to prevent the algorithm from ever converging. 
\end{proof}

Observe that the vector chosen by some node at the end of the first round of Algorithm~\ref{alg:md} is a $2$-approximation of the geometric median of the non-faulty nodes. This is because the computed vector is inside $\geoSet$ and thus at most $2\cdot\radEnc$ away from the non-faulty geometric median. Thus, even though the \MD algorithm for the geometric median aggregation rule does not converge, the locally chosen vectors by each node are still representative. In addition, this also means that the \MD algorithm for the geometric median computes a $2$-approximation of the geometric median of the non-faulty nodes in the centralized collaborative learning setting.

Finally, we consider Krum~\cite{blanchard2017machine,byzantine-primer}. This aggregation rule is not used as an approximate agreement algorithm in the literature, since median/medoid-based approximation algorithms would not converge (due to a similar argument to Lemma~\ref{lem:md_no_convergence}). 
In the following, we show that within one round of Krum (one single application of Equation~\ref{eq:Krum}) a server cannot compute a bounded approximation of the geometric median.

\begin{theorem} \label{Krum-unbounded-geom}
The approximation ratio of the Krum aggregation rule is unbounded.
\end{theorem}
\begin{proof}
    Consider a setting where Byzantine parties do not send any vectors to the correct nodes. That is, the received $n-t$ vectors are all from non-faulty nodes. Assume further a general case where the medoid of the received $n-t$ vectors does not correspond to the geometric median of these vectors. 

    Note that due to the fact that no Byzantine vectors are present in the calculation, the ball of all possible geometric medians is a single point. Since the medoid and the geometric median are assumed to not be the same point, the approximation ratio in this case is unbounded.
\end{proof}

Observe that the result from Theorem~\ref{Krum-unbounded-geom} also holds for Multi-Krum: Since the server receives exactly $n-t$ vectors in this example, every computed medoid will be computed on the same set of $n-t$ vectors. Thus, we have $\text{Multi-Krum}_q(v_1,\ldots, v_k) = \text{Krum}(v_1,\ldots, v_k)$, and the same unbounded approximation ratio holds for Multi-Krum.

\subsection{Hyperbox approach}

Let us now consider an adaptation of the hyperbox algorithm as a candidate for solving multidimensional approximate agreement for the geometric median aggregation rule. 

\begin{algorithm}[tbh]
\caption{Synchronous approximate agreement with hyperbox validity and resilience $t<n/3$}
\label{alg:hyperbox approach}
\begin{algorithmic}[1]
    \State In each round $r=1,2,\ldots$, each node $u_i, \forall i \in [n]$ with input vector $v_i$ executes the following:
        
            \State Broadcast $v_i$ reliably to all nodes
            \State Reliably receive up to $n$ messages $M_i = \{v_j, j\in [n]\}$
            \State Compute $\trueBox_i$ from $M_i$ by excluding $|M_i|-(n-t)$ values on each side 
            \State Compute $\medianBox_i$ from $M_i$ 
            \State Set $v_i$ to $\midpointFunction(\trueBox_i \cap \medianBox_i)$.
        
\end{algorithmic}
\end{algorithm}

Contrary to using the safe area algorithm, \Cref{alg:hyperbox approach} gives a bounded approximation of the true geometric median. And contrary to \Cref{alg:md}, it is guaranteed to converge and hence solved the multidimensional approximate agreement problem.

\begin{theorem}
    \Cref{alg:hyperbox approach} converges and its approximation ratio is upper bounded by $2\sqrt{d}$. 
\end{theorem}
\begin{proof}
    First we need to show that the algorithm can run, i.e. that at any round $r$, the intersection of $\trueBox_i$ and $\medianBox_i$ is non empty for every node $u_i$. We then need to show that the algorithm converges. And lastly we need to show that each node terminates on a vector that is at most $2\sqrt{d}\cdot \radEnc$ away from the true geometric median.  

    \textbf{Hyperbox intersection in each round.} 
Let us fix a coordinate $k\in [d]$ and let $v_1, \dots, v_m$ be the vectors received by node $u_i$ (hence, $m\geq n-t$). We now define $\phi_i: [m] \to [m]$ a bijection s.t.  $v_{\phi_i(j_1)} [k]\leq v_{\phi_i(j_2)} [k], \forall j_1, j_2\in [m]$. 
    The locally trusted hyperbox in coordinate $k$ is defined as
    \begin{align*}
        \trueBox_i[k] = \left[v_{\phi_i(t+1)}[k], v_{\phi_i(n-t)}[k]\right].
    \end{align*}
    Consider the two following elements of the set of possible geometric medians of $n-t$ vectors: 
    \begin{align*}
        &g_{\alpha} = \geometricMedian(v_{\phi_i(1)}, \dots, v_{\phi_i(n-t)}) \text{ and } \\
        &g_{\beta} = \geometricMedian(v_{(\phi_i(t+1)}, \dots, v_{\phi_i(m)}).
    \end{align*}
    Then, by definition of the geometric median and since $v_{\phi_i(j)}[k]$ are in increasing order, $g_{\alpha}[k] \leq v_{\phi_i(n-t)}[k]$. 
    Similarly, $g_{\beta}[k] \geq v_{\phi_i(t+1)}[k]$. 
    Moreover, the interval spanned by $g_{\alpha}[k]$ and $g_{\beta}[k]$ is included in $\medianBox_i$. 
    Hence, $\medianBox_i\cap \trueBox_i\neq \emptyset$. 

    \textbf{Algorithm convergence.} 
We denote $\trueBox = \trueBox^1$ the smallest hyperbox containing all true vectors, and $\trueBox^{r+1}$ the smallest hyperbox containing all the vectors computed by correct nodes in round $r$, which represents the input in round $r+1$. 
In round $r$, a node $u_i$ computes $\trueBox_i^r\subseteq \trueBox^{r}$, and then picks a vector in this hyperbox. 

To prove convergence, we will show that for any correct nodes $u_i$ and $u_j$ in any round $r\geq 1$:
\begin{align*}
\bigl|\midpointFunction(\trueBox_i^r\cap \medianBox_i^r) - \midpointFunction(\trueBox_j^r\cap \medianBox_j^r)\bigr|
\leq \frac{1}{2}\cdot \longestEdge(\trueBox^{r}).
\end{align*}

First, $\trueBox_i^r\subseteq \trueBox^{r}$. We show this for a fixed coordinate $k\in[d]$, which implies the general result since we work with hyperboxes. 
Indeed, when round $r$ starts, $\trueBox^{r}[k]$ is the smallest interval containing all $v_j[k]$ where $u_j$ are true nodes.
If the interval $\trueBox_i^r[k]$ does not contain any Byzantine values, then $\trueBox_i^r[k]\subseteq \trueBox^{r}[k]$ by definition. 
Otherwise, $\trueBox_i^r[k]$ contains at least one Byzantine value. Thus, when the minimum and maximum values were trimmed on each side of the interval, this Byzantine value remained in $\trueBox_i^r[k]$, and at least two true values were removed instead (one on each side of the interval). Therefore, $\trueBox_i^r[k]$ is included in an interval bounded by two true values, which is by definition included in $\trueBox^{r}[k]$.
Since this holds for each coordinate $k$, it also holds that $\trueBox_i^r\subseteq \trueBox^{r}$. 

Second, We define $\psi: [n]\times [n] \to [n]\times [n]$ a bijection s.t.  $v_{\psi(i,j_1)} [k]\leq v_{\psi(i,j_2)} [k], \forall i, j_1, j_2\in [n]$ where $v_{i, j}$ is the vector received by node $u_i$ from node $u_j$.   
For node $u_i$, $\trueBox_i^r[k] = \bigl[v_{\psi(i, m_i-(n-t)+1)}[k],\allowbreak v_{\psi(i, n-t)}[k]\bigr]$, and for node $u_j$ $\trueBox_j^r[k] = \bigl[v_{\psi(j, m_j-(n-t)+1)}[k], v_{\psi(j, n-t)}[k]\bigr]$, where $m_i$ and $m_j$ are the number of messages received by nodes $u_i$ and $u_j$ respectively. Given all projections on true vectors in coordinate $k$, denote $t_\ell$ the minimum and $t_u$ the maximum of those values. 
Since $\trueBox_i^r[k]\subseteq \trueBox^{r}[k]$, $v_{\psi(i, m_i-(n-t)+1)}[k]\geq t_\ell$ and $v_{\psi(i, n-t)}[k]\leq t_u$. 
Similarly, $v_{\psi(j, m_i-(n-t)+1)}[k]\geq t_\ell$ and $v_{\psi(j, n-t)}[k]\leq t_u$. 
Moreover, since projections of Byzantine vectors can be inside $\trueBox[k]$, the intervals $\trueBox_i[k]$ and $\trueBox_j[k]$ can be computed by removing up to $m_i-(n-t)$ and $m_j-(n-t)$ true values on each side respectively. 
Suppose w.l.o.g. that $m_i\leq m_j$, and denote $v^*_1, \dots v^*_{n-f}$ the vectors of true nodes.
Let us define a bijection $\lambda: [n-f]\to [n-f]$ s.t. $v^*_{\lambda(j_1)}[k]\leq v^*_{\lambda(j_2)}[k], \forall j_1, j_2\in [n-f]$. 
Let $t_{1}\coloneqq v^*_{\lambda(m_i-(n-t)+1)}[k]$ and $t_{2}\coloneqq v^*_{\lambda(2n-f-t-m_i)}[k]$ , these true values are necessarily inside both $\trueBox_i[k]$ and $\trueBox_j[k]$. 
Then, $v_{\psi(i, m_i-(n-t)+1)}[k]\leq t_{1}$ and $v_{\psi(i, n-t)}[k] \geq t_{2}$. 
Similarly, $v_{\psi(j, m_i-(n-t)+1)}[k]\leq t_{1}$ and $v_{\psi(j, n-t)}[k] \geq t_{2}$. 
We can now upper bound the distance between the computed midpoints of the nodes $u_i$ and $u_j$: 
\begin{align*}
    \Bigl|\midpointFunction\bigl(\trueBox_i^r[k]\bigr) - \midpointFunction\bigl(\trueBox_j^r[k]\bigr)\Bigr|
    \leq \frac{t_u-t_{1}}{2} - \frac{t_{2} - t_\ell}{2} 
    \leq \frac{t_u-t_\ell}{2}
    \leq \frac{\longestEdge(\trueBox^{r})}{2}.
\end{align*}
This inequality holds for every pair of nodes $u_i$ and $u_j$ and thus, for each coordinate $k$, we get
\begin{align*}
    \max_{i, j\in [n]}\bigl|\midpointFunction\bigl(\trueBox_i^r[k]\bigr) - \midpointFunction\bigl(\trueBox_j^r[k]\bigr)\bigr| &\leq  \longestEdge(\trueBox^{r})/2\\
     \Leftrightarrow\ \  \longestEdge\bigl(\trueBox^{r+1}[k]\bigr) &\leq  \longestEdge(\trueBox^{r})/2.
\end{align*}

After $R$ rounds, $\longestEdge(\trueBox^R)\leq \frac{1}{2^R}\cdot \longestEdge(\trueBox)$ holds. 
Since there exists $R\in \mathbb{N}$ s.t.\ $\frac{1}{2^R}\cdot \longestEdge(\trueBox)\leq \varepsilon$, the algorithm converges.

    \textbf{Approximation ratio.}
First, observe that the radius of the minimum covering ball of $\geoSet$ is always at least  $\max_{x, y\in \geoSet}\bigl(\distance(x, y)\bigr)/2$.

Next, let us upper bound the distance between $\trueMedian$ and the furthest possible point from it inside $\medianBox$. W.l.o.g., we assume that $\max_{x, y\in \geoSet}\bigl(\distance(x, y)\bigr) = 1$.
We consider the relation between the minimum covering ball and $\medianBox$.
Observe that each face of the hyperbox $\medianBox$ has to contain at least one point of $\geoSet$. 
If $\medianBox$ is contained inside the ball, i.e.\ if the vertices of the hyperbox lie on the ball surface, the computed approximation of \Cref{alg:hyperbox approach} would always be optimal. 

The worst case is achieved if the ball is (partly) contained inside $\medianBox$. Then, the optimal solution might lie inside $\medianBox$ and the ball, while the furthest node may lie on one of the vertices of $\medianBox$ outside of the ball. The distance of any node from $\trueMedian$ in this case is upper bounded by the diagonal of the hyperbox. Since the longest distance between any two points was assumed to be $1$, the hyperbox is contained in a unit cube. 
$\medianBox$ can be the unit cube itself and thus the largest distance between two points of $\medianBox$ is at most $\sqrt{d}$.  

The approximation ratio of \Cref{alg:hyperbox approach} can hence be upper bounded by:
\begin{align*}
\frac{\max_{x\in \medianBox}\bigl(\distance(\trueMedian, x)\bigr)}{\radEnc} 
&\le 2\cdot\frac{\max_{x\in \medianBox}\bigl(\distance(\trueMedian, x)\bigr)}{\max_{x, y\in \geoSet}\bigl(\distance(x, y)\bigr)} 
\le 2\cdot\frac{\sqrt{d}}{1} 
= 2\sqrt{d}.
\end{align*}

\end{proof}

\section{Empirical evaluation}\label{sec: practice}
Given our formal analysis of the single-round aggregation, we now perform an empirical evaluation of the algorithms to understand how the convergence of the approximate agreement algorithms influences the convergence of the machine learning model. In addition, we want to investigate how the approximation ratio within one learning round relates to the final accuracy of the model. In the empirical results, we apply the aggregation rules discussed in this paper in every learning round. 
This means that it is hard to link the empirical results to the theoretical results of this paper, which only evaluate the quality of the aggregation vector in one single learning round.

\subsection{Methodology}
A centralized and a decentralized collaborative learning model for solving classification tasks are implemented in Python using the Tensorflow library -- an end-to-end platform for solving machine learning tasks. 

The models are evaluated on the MNIST from Kaggle\footnote{\url{https://www.kaggle.com/datasets/scolianni/mnistasjpg}, accessed on 27.02.2025} and CIFAR10 dataset \footnote{\url{https://www.cs.toronto.edu/~kriz/cifar.html}, accessed on 27.02.2025}. The MNIST dataset contains 42,000  $28 \times 28$ images of handwritten digits, whereas CIFAR10 has 60,000 $32 \times 32$ color images in 10 classes, out of which 50,000 are training images and 10,000 test images. The MNIST dataset is split into train and test data with ratio $9:1$. 
We consider the uniform and 2 cases of non-uniform data distributions.
The first case is mild heterogeneity, where each class from the train dataset is split into $10$ parts, where $8$ parts contain $10\%$ of the class, one part $5\%$ and one part $15\%$ of the class. The second case is extreme heterogeneity, also known as 2-class heterogeneity. The dataset is sorted and split into 20 pieces. Each client gets randomly 2 parts of the data, so that each client has up to 2 classes of data in its local dataset. Note that the scenarios where clients have different local dataset sizes are not taken into account, as Byzantine clients could exploit this variation to their advantage. 

For stochastic gradient computation, a random batch of data is chosen and loss is computed using categorical cross-entropy. The gradient estimate is calculated using \textit{tape.gradient} with respect to the model's trainable variables.

The underlying neural network for solving the image classification task on MNIST dataset is a MultiLayer Perceptron (MLP) with 3 layers. The learning rate is set to $\eta = 0.01$ and the decay is calculated with respect to the number of global communication rounds (epochs), i.e. $decay = \frac{\eta}{rounds}$. The approach for decaying over global instead of local (current) epoch was proposed in \cite{zhao2018federated}. 

For the CIFAR10 dataset we implemented CifarNet, a medium-sized convolutional neural network with thousands of trainable parameters and the ability to capture spatial relationships in colored images. 

In the experiments, we set the number of clients to $n=10$ and number of Byzantine nodes to $f=1$ and $f=2$. We consider the sign flip attack \cite{park2024signsgd} . 
The attack consists of $f$ Byzantine clients computing their gradients and then inverting their sign. Flipped gradients are sent to either the central server or all other clients, depending on the architecture. Such an attack is difficult to detect and thus the Byzantine gradient is used in computations in the same way as other local gradients. 

\subsection{Empirical results}

\begin{figure}
    \centering
    \includegraphics[width=0.7\linewidth]{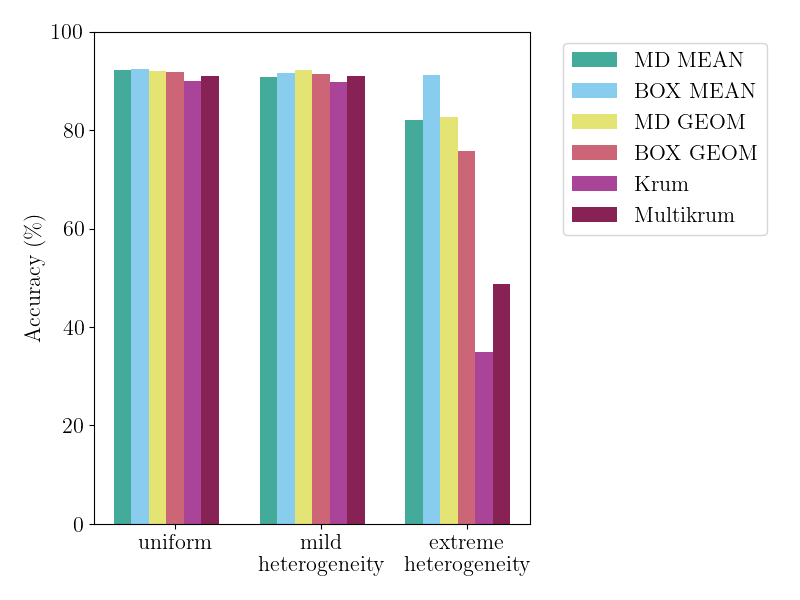}
    \caption{Centralized collaborative learning with $f=1$ on MLP architecture and MNIST dataset, under different heterogeneity}
    \label{fig:mlpf1}
\end{figure}
In the following, we evaluate mean and geometric median using Algorithm~\ref{alg:md} (\mdmedian) and Algorithm~\ref{alg:hyperbox approach} (\boxmedian) under sign flip attack in centralized and decentralized collaborative learning with MNIST and CIFAR10 dataset. For the geometric median computation, the Weiszfeld algorithm is used \cite{weiszfeld1937point}. In the centralized setting, we additionally test Krum and Multi-Krum with $q=3$ defined in \ref{krum}. 
For comparison, we also evaluate the hyperbox algorithm (\boxmean) and the minimum diameter averaging algorithm (\mdmean), described in \Cref{subsec: multidimApproxAgreement}.

Figure \ref{fig:mlpf1} illustrates achieved accuracy of different aggregation algorithms in the centralized collaborative learning model on MLP architecture using the MNIST dataset. We set $f=1$.
Firstly, all methods perform better with uniform and mild heterogeneous data, compared to extreme heterogeneous data. Algorithms \mdmean, \mdmedian, \boxmean and \boxmedian achieve over $91\%$ accuracy with uniform and mild heterogeneous data distribution. Krum and Multi-Krum perform well on uniform and mildly heterogeneous data but fail to exceed $50\%$ accuracy in the extremely heterogeneous setting. This is because both methods rely on selecting and averaging a small number of input points ($q=1$ or $q=3$), which, in extreme heterogeneity, are too far apart to provide a reliable estimate.

\begin{figure}
    \begin{subfigure}[b]{0.49\textwidth}
    \includegraphics[width=0.99\linewidth]{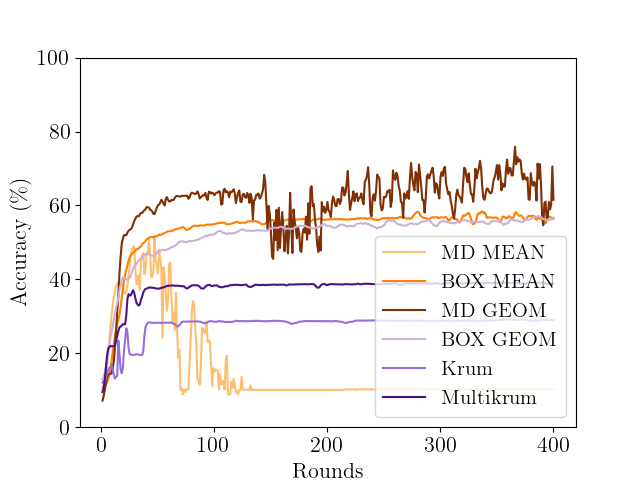}
    \caption{MLP on MNIST dataset with $f=2$ and \\extreme heterogeneity}
    \label{fig:mlpf2}
    \end{subfigure}
    \begin{subfigure}[b]{0.49\textwidth}
    \includegraphics[width=0.99\linewidth]{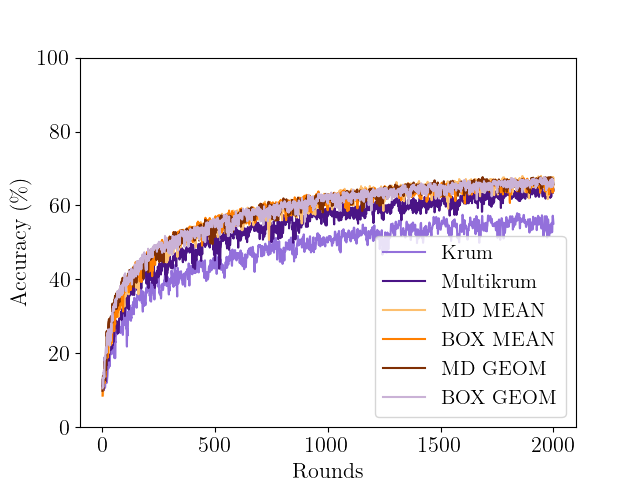}
    \caption{CifarNet on CIFAR10 dataset with $f=1$ and mild heterogeneity}
    \label{fig:cifarnet}
    \end{subfigure}
    \caption{Centralized collaborative learning on MLP and CifarNet, using MNIST and CIFAR10 dataset}
    \label{fig:centralizedcollab}
\end{figure}
Figure \ref{fig:centralizedcollab} illustrates centralized collaborative learning in a more extreme scenario, on MLP and MNIST dataset in Figure \ref{fig:mlpf2}, and on CifarNet and CIFAR10 dataset in Figure \ref{fig:cifarnet}. In Figure \ref{fig:mlpf2} we consider the extreme heterogeneous setting with 2 Byzantine sign flip attacks. It can be observed that \mdmean fails to converge and \mdmedian is unstable. Krum and Multikrum converge with low accuracy of $30\%$ and $39\%$, respectively. \boxmean and \boxmedian converge and score $57\%$ accuracy. The algorithm \mdmedian achieves the best accuracy, illustrating the fact that this algorithm has the best approximation ratio in the centralized setting (\Cref{sec: theory}). 

Figure \ref{fig:cifarnet} shows centralized collaborative learning on CifarNet, evaluated on CIFAR10 dataset. Since CifarNet is more complex than the MLP and CIFAR-10 consists of colored images, unlike MNIST, the accuracy of all methods drops to at most $70\%$. Due to its complexity, CifarNet also requires more communication rounds to converge, than the MLP. Algorithms \boxmedian, \boxmean, \mdmedian and \mdmean achieve over $67\%$ accuracy, whereas Multikrum scores $64\%$. Krum performs significantly worse and achieves $55\%$ accuracy.

\begin{figure}[tb]
    \begin{subfigure}[b]{0.49\textwidth}
    \includegraphics[width=0.99\linewidth]{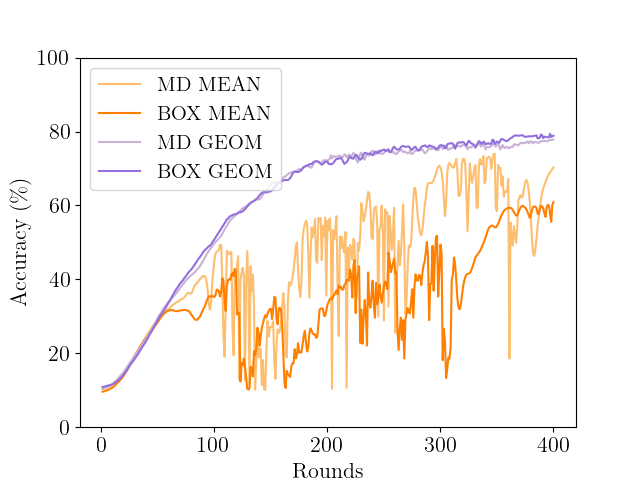}
    \caption{MLP with $f=1$}
    \label{mlp-decentralized-f1}
    \end{subfigure}
    \begin{subfigure}[b]{0.49\textwidth}
    \includegraphics[width=0.99\linewidth]{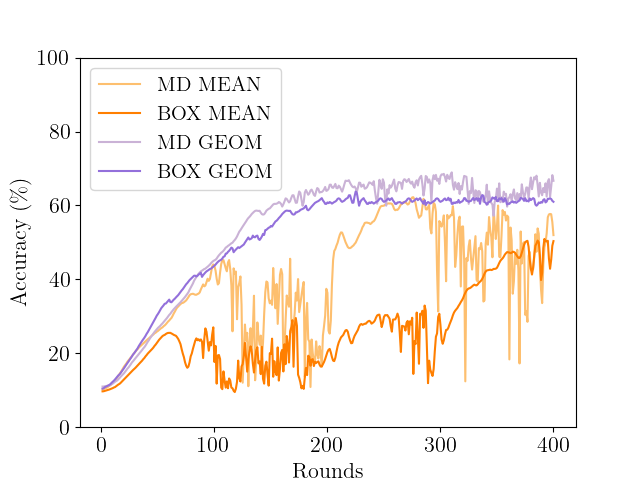}
    \caption{MLP with $f=2$}
    \label{mlp-decentralized-f2}
    \end{subfigure}
    \caption{Decentralized collaborative learning model on MLP architecture with mild heterogeneous data}
   \label{fig:decentralized}
\end{figure}

In Figure \ref{mlp-decentralized-f1} we consider decentralized collaborative learning model with MLP architecture and $f=1$. It can be observed that mean-based aggregation rules do not converge under the sign flip attack. Upon deeper analysis, some local models after round $100$ learn well and some do not. This happens because models agree on vectors that do not suit them well. Note that, clients update their models after performing aggregation rules.
The local gradients clients compute in the next round are also bad, since the parameters of the model worsened in the previous round. 
In contrast, \mdmedian and \boxmedian both converge and achieve $77.8\%$ and $78.8 \%$ accuracy, respectively. 

Figure \ref{mlp-decentralized-f2} shows decentralized collaborative learning with MLP architecture under 2 Byzantine sign flip attacks. \mdmean and \boxmean fail to converge, which correlates with the result from Figure \ref{mlp-decentralized-f1}. \mdmedian scores $65\%$ but is considered to be unstable, whereas \boxmedian seems to converge with $62\%$ accuracy. Figure \ref{fig:decentralized} highlights the advantages of geometric median-based aggregation algorithms compared to mean-based aggregation algorithms.

\subsection{Discussion}

We first discuss the centralized collaborative learning setting. In this setting, the results of the algorithms differ for extremely heterogeneous data under two sign flip failures, see Figure \ref{fig:mlpf2}. In particular, \mdmedian achieves better accuracy than \boxmedian, and both these algorithms outperform Multi-Krum and Krum. This reflects our theoretical results, where we showed that \mdmedian computes a $2$-approximation of the geometric median in the centralized collaborative learning setting, the \boxmedian algorithm a $\sqrt{10}\approx 3.16$-approximation, while Krum and Multi-Krum have unbounded approximation ratios in the worst case. 
For uniform and mildly heterogeneous data distributions we do not find such a difference in accuracies, see Figure \ref{fig:mlpf1}. 

In the decentralized collaborative learning setting, our experimental results indicate that the convergence of the approximate agreement subroutine in one learning round does not influence the convergence of the machine learning model. Recall that the \MD algorithm for the geometric median does not converge, and there can be two groups of nodes with vastly different gradients that use the respective gradient to update their local models. One would expect that such a setting would prevent \mdmedian from converging. Such a scenario does not seem to appear under sign flip attacks in practice. The convergence of the model for \mdmedian (see Figure \ref{mlp-decentralized-f1}) suggests that small discrepancies among the gradients in one learning round do not influence the convergence of the ML model. Moreover, the approximated aggregation vector (mean or median) seems to have a more important role in the decentralized setting. Median-based approaches outperform mean-based approaches under mildly heterogeneous data distribution, see Figure \ref{fig:decentralized}. Under extremely heterogeneous data distribution, however, both aggregation rules fail, suggesting that a different approach may be necessary in this case.

\section{Related work}\label{sec: related_work}

Federated learning was introduced by McMahan et al.~\cite{google-federated, fedavg} for supervised learning, where the data of clients is either sensitive or too large to be stored in data center. They consider an unbalanced dataset with non-i.i.d. data in a massively distributed setting, with clients which have little data on average. The training is arranged in a server-client architecture: the server manages model parameters and the client handles the training. While being robust, this paper and much of the follow-up work~\cite{konečný2015federated,konečný2017federated,pmlr-v54-mcmahan17a} do not tolerate malicious attacks.

\textbf{Byzantine attacks in federated learning. }
Byzantine attacks are defined as arbitrary worst-case attacks in the system. In the literature, however, often specific malicious behavior is considered that can harm the training process in machine learning.
Blanchard et al.~\cite{blanchard2017machine} show that federated learning approaches based on linear combinations of the input cannot tolerate a single Byzantine failure. They consider a single attacker who knows the local updates of all benign clients. Such an attacker can set its update to the opposite of the combined normal updates, thus preventing convergence.

Jere et al.~\cite{9308910} provide a survey of malicious attacks in federated learning. They divide the attacks into model poisoning\cite{9464278}, comprising of label flipping and backdoor attacks~\cite{pmlr-v108-bagdasaryan20a}; and data poisoning attacks, including gradient manipulation~\cite{247652,blanchard2017machine} and training rule manipulation~\cite{pmlr-v97-bhagoji19a}. 

Shi et al.~\cite{shi2022challenges} make a similar classification to \cite{9308910} and propose the \emph{weight attack}, which bypasses existing defense schemes. The idea is to exploit the fact that a central entity has no effective means to check the size and quality of one's data. Therefore, Byzantine clients can claim to have a larger dataset than the rest and gain high weight parameters. This attack is not considered in our work, as we assume that all clients in the system have the same amount of data.

The main attack considered in this paper is the sign flip attack. In~\cite{jungle}, a multiplicative factor is added to the sign flip attack. While this attack aims to increase the harm with an increasing multiplicative factor, it also makes it easier to remove the attacker from the training over time.
Park and Lee~\cite{park2024signsgd} consider the sign flip attack in a more powerful setting: they study the \textit{signSGD} algorithm~\cite{jin2020signSGD, bernstein2018signsgd}, where instead of transmitting gradients, only signs of gradients are exchanged.

\textbf{Byzantine-tolerant federated learning. }
The first Byzantine-tolerant federated learning algorithms address Byzantine behavior of clients, but they rely on a trusted central entity~\cite{feng2015distributedrobustlearning,chen_2017,blanchard2017machine,Su-Vaidya-2016,cwmDEF,pmlr-v139-karimireddy21a}. 

The work of El-Mhamdi et al. \cite{el2020genuinely} explores the general Byzantine-tolerant distributed machine learning problem, where no individual component can be trusted. Their idea is to replicate the server onto multiple nodes, which appear as one central entity to the user, thus making the central entity Byzantine-tolerant as well. 

The first fully decentralized Byzantine-tolerant federated learning model was proposed by El-Mhamdi et al.~\cite{jungle}. The authors first define the collaborative learning model in detail and show the equivalence of the collaborative learning problem and averaging agreement. Additionally, two optimal algorithms for averaging agreement are implemented~\cite{guerraoui2021garfield}: minimum-diameter based algorithm, which asymptotically optimal with respect to correctness, when nearly all nodes are non-faulty, and trimmed mean algorithm with optimal Byzantine resilience $t<\frac{n}{3}$. 

Guerraoui et al. standardize the study of Byzantine machine learning and provide an overview 
of shortcomings of widely used approaches in a survey~\cite{byzantine-primer} and a follow-up work~\cite{farhadkhani2024brief}.

\textbf{Aggregation rules in federated learning. }
Besides using the mean as an aggregation function, many other aggregation rules have been considered in the literature~\cite{byzantine-primer}. 

Pillutla et al.~\cite{pillutla2022robust} use the geometric median~\cite{kuhn-geom-history} as an aggregation function. 
Despite its simple definition, the geometric median is hard to compute~\cite{BAJAJ198699} and requires an approximation algorithm. To this end, the Weiszfeld algorithm for computing geometric median~\cite{weiszfeld1937point,weiszfeld2009point} is used in \cite{pillutla2022robust} and in this work.

El-Mhamdi et al.\cite{guerraoui2018medoid} propose to use geometric medoids. Similar to the geometric median, the geometric medoid minimizes the sum of distances to all points, but its value is among input vectors. Naturally, medoid is easier to compute than the geometric median, since it requires testing every input vector regarding the distances to other points. However, in their experiments, medoid failed to produce a useful model.

Another aggregation rule named \emph{Krum} was proposed by Blanchard et at. \cite{blanchard2017machine}. Krum is calculated as the vector that minimizes the sum of squared distances to its $n-f$ closest vectors. Krum was proposed as an alternative to looking at all possible subsets of size $n-f$ and then considering the one with minimum diameter, as this approach has exponential runtime. In their experiments, Krum is proven to be robust against Byzantine attacks compared to the classical averaging aggregation functions. 

\textbf{Byzantine Agreement.}
Byzantine agreement was originally introduced by Lamport~\cite{10.1145/357172.357176} to deal with unpredictable system faults. It requires the nodes to agree on the same value (agreement) within finite time (termination) while outputting a non-trivial solution (validity). Multidimensional approximate agreement~\cite{10.1145/2484239.2484256,10.1145/2488608.2488657,attiya_et_al:LIPIcs.OPODIS.2022.6,10.1145/3558481.3591105} generalizes the input values of the nodes to vectors and relaxes the agreement condition such that the nodes can terminate when the output vectors are in the vicinity of each other. This allows one to speed up the communication-intensive distributed agreement algorithms.

El-Mhamdi et al.~\cite{jungle} draw a first connection between approximate agreement and distributed collaborative learning. They show that averaging agreement, defined as approximate agreement where the output vectors are close to the mean of the benign vectors, is equivalent to distributed collaborative learning. Their distance between the output vectors is bounded by the maximum distance between the furthest benign input vectors. 
Cambus and Melnyk \cite{centroid-paper} refine the idea to approximate the mean in the setting of approximate agreement. They introduce an approximation measure used in this paper. This approximation ratio allows one to compare the output vectors of an approximate agreement algorithm to a solution given by an optimal algorithm that cannot identify Byzantine values.

\section{Conclusion}\label{sec: conclusion}

This paper analyzed the geometric median as the aggregation rule for fully distributed Byzantine-tolerant collaborative learning. The theoretical analysis showed that using the geometric median directly, or in combination with the safe area or the minimum diameter, does not lead to convergence of the agreement routine or to a reasonable approximation of the geometric median. The hyperbox approach in combination with the geometric median was presented as a possible approach that provides the desirable theoretical guarantees. The practical evaluation revealed that approaches based on the geometric median provide more stable solutions under the sign flip attack in the distributed collaborative learning setting. In the future, it would be interesting to investigate whether \mdmedian also converges under Byzantine behavior that uses information from multiple learning rounds. In addition, new aggregation rules besides the mean and the median should be investigated for distributed collaborative learning under extremely heterogeneous data distributions.

\bibliography{references}

\begin{thebibliography}{51}
\providecommand{\natexlab}[1]{#1}
\providecommand{\url}[1]{\texttt{#1}}
\expandafter\ifx\csname urlstyle\endcsname\relax
  \providecommand{\doi}[1]{doi: #1}\else
  \providecommand{\doi}{doi: \begingroup \urlstyle{rm}\Url}\fi

\bibitem[Attiya and Ellen(2023)]{attiya_et_al:LIPIcs.OPODIS.2022.6}
H.~Attiya and F.~Ellen.
\newblock {The Step Complexity of Multidimensional Approximate Agreement}.
\newblock In \emph{26th International Conference on Principles of Distributed Systems (OPODIS 2022)}, 2023.
\newblock \doi{10.4230/LIPIcs.OPODIS.2022.6}.

\bibitem[Bagdasaryan et~al.(2020)Bagdasaryan, Veit, Hua, Estrin, and Shmatikov]{pmlr-v108-bagdasaryan20a}
E.~Bagdasaryan, A.~Veit, Y.~Hua, D.~Estrin, and V.~Shmatikov.
\newblock How to backdoor federated learning.
\newblock In \emph{Proceedings of the Twenty Third International Conference on Artificial Intelligence and Statistics}, volume 108 of \emph{Proceedings of Machine Learning Research}, pages 2938--2948. PMLR, 26--28 Aug 2020.
\newblock URL \url{https://proceedings.mlr.press/v108/bagdasaryan20a.html}.

\bibitem[Bajaj(1986)]{BAJAJ198699}
C.~Bajaj.
\newblock Proving geometric algorithm non-solvability: An application of factoring polynomials.
\newblock \emph{Journal of Symbolic Computation}, 1986.
\newblock ISSN 0747-7171.
\newblock \doi{https://doi.org/10.1016/S0747-7171(86)80015-3}.
\newblock URL \url{https://www.sciencedirect.com/science/article/pii/S0747717186800153}.

\bibitem[Bernstein et~al.(2018)Bernstein, Zhao, Azizzadenesheli, and Anandkumar]{bernstein2018signsgd}
J.~Bernstein, J.~Zhao, K.~Azizzadenesheli, and A.~Anandkumar.
\newblock signsgd with majority vote is communication efficient and fault tolerant.
\newblock \emph{arXiv preprint arXiv:1810.05291}, 2018.

\bibitem[Bhagoji et~al.(2019)Bhagoji, Chakraborty, Mittal, and Calo]{pmlr-v97-bhagoji19a}
A.~N. Bhagoji, S.~Chakraborty, P.~Mittal, and S.~Calo.
\newblock Analyzing federated learning through an adversarial lens.
\newblock In \emph{Proceedings of the 36th International Conference on Machine Learning}, volume~97 of \emph{Proceedings of Machine Learning Research}, pages 634--643. PMLR, 09--15 Jun 2019.
\newblock URL \url{https://proceedings.mlr.press/v97/bhagoji19a.html}.

\bibitem[Blanchard et~al.(2017)Blanchard, El~Mhamdi, Guerraoui, and Stainer]{blanchard2017machine}
P.~Blanchard, E.~M. El~Mhamdi, R.~Guerraoui, and J.~Stainer.
\newblock Machine learning with adversaries: Byzantine tolerant gradient descent.
\newblock \emph{Advances in neural information processing systems}, 30, 2017.
\newblock URL \url{https://proceedings.neurips.cc/paper_files/paper/2017/file/f4b9ec30ad9f68f89b29639786cb62ef-Paper.pdf}.

\bibitem[Bottou et~al.(2018)Bottou, Curtis, and Nocedal]{bottou2018optimization}
L.~Bottou, F.~E. Curtis, and J.~Nocedal.
\newblock Optimization methods for large-scale machine learning.
\newblock \emph{SIAM review}, 60\penalty0 (2):\penalty0 223--311, 2018.
\newblock \doi{10.1137/16M1080173}.

\bibitem[Boyd and Vandenberghe(2004)]{boyd2004convex}
S.~Boyd and L.~Vandenberghe.
\newblock \emph{Convex optimization}.
\newblock Cambridge university press, 2004.

\bibitem[Boyd et~al.(2011)Boyd, Parikh, Chu, Peleato, Eckstein, et~al.]{boyd2011distributed}
S.~Boyd, N.~Parikh, E.~Chu, B.~Peleato, J.~Eckstein, et~al.
\newblock Distributed optimization and statistical learning via the alternating direction method of multipliers.
\newblock \emph{Foundations and Trends{\textregistered} in Machine learning}, 2011.
\newblock \doi{10.1561/2200000016}.

\bibitem[Bracha(1987)]{BRACHA1987130}
G.~Bracha.
\newblock Asynchronous byzantine agreement protocols.
\newblock \emph{Information and Computation}, 75\penalty0 (2):\penalty0 130--143, 1987.
\newblock ISSN 0890-5401.
\newblock \doi{https://doi.org/10.1016/0890-5401(87)90054-X}.
\newblock URL \url{https://www.sciencedirect.com/science/article/pii/089054018790054X}.

\bibitem[Cambus and Melnyk(2023)]{centroid-paper}
M.~Cambus and D.~Melnyk.
\newblock Improved solutions for multidimensional approximate agreement via centroid computation, 2023.
\newblock URL \url{https://arxiv.org/abs/2306.12741}.

\bibitem[Chen et~al.(2017)Chen, Su, and Xu]{chen_2017}
Y.~Chen, L.~Su, and J.~Xu.
\newblock Distributed statistical machine learning in adversarial settings: Byzantine gradient descent.
\newblock \emph{Proc. ACM Meas. Anal. Comput. Syst.}, 2017.
\newblock \doi{10.1145/3154503}.

\bibitem[El~Mhamdi et~al.(2018)El~Mhamdi, Guerraoui, and Rouault]{guerraoui2018medoid}
E.~M. El~Mhamdi, R.~Guerraoui, and S.~Rouault.
\newblock The hidden vulnerability of distributed learning in {B}yzantium.
\newblock In \emph{Proceedings of the 35th International Conference on Machine Learning}, volume~80 of \emph{Proceedings of Machine Learning Research}, pages 3521--3530. PMLR, 10--15 Jul 2018.
\newblock URL \url{https://proceedings.mlr.press/v80/mhamdi18a.html}.

\bibitem[El-Mhamdi et~al.(2020)El-Mhamdi, Guerraoui, Guirguis, Hoang, and Rouault]{el2020genuinely}
E.-M. El-Mhamdi, R.~Guerraoui, A.~Guirguis, L.~N. Hoang, and S.~Rouault.
\newblock Genuinely distributed byzantine machine learning.
\newblock In \emph{Proceedings of the 39th Symposium on Principles of Distributed Computing}, PODC '20, 2020.
\newblock \doi{10.1145/3382734.3405695}.

\bibitem[El-Mhamdi et~al.(2021)El-Mhamdi, Farhadkhani, Guerraoui, Guirguis, Hoang, and Rouault]{jungle}
E.-M. El-Mhamdi, S.~Farhadkhani, R.~Guerraoui, A.~Guirguis, L.-N. Hoang, and S.~Rouault.
\newblock Collaborative learning in the jungle (decentralized, byzantine, heterogeneous, asynchronous and nonconvex learning).
\newblock NIPS '21, 2021.

\bibitem[Fang et~al.(2020)Fang, Cao, Jia, and Gong]{247652}
M.~Fang, X.~Cao, J.~Jia, and N.~Gong.
\newblock Local model poisoning attacks to {Byzantine-Robust} federated learning.
\newblock In \emph{29th USENIX Security Symposium (USENIX Security 20)}. USENIX Association, Aug. 2020.
\newblock ISBN 978-1-939133-17-5.
\newblock URL \url{https://www.usenix.org/conference/usenixsecurity20/presentation/fang}.

\bibitem[Farhadkhani et~al.(2024)Farhadkhani, Guerraoui, Gupta, and Pinot]{farhadkhani2024brief}
S.~Farhadkhani, R.~Guerraoui, N.~Gupta, and R.~Pinot.
\newblock Brief announcement: A case for byzantine machine learning.
\newblock In \emph{Proceedings of the 43rd ACM Symposium on Principles of Distributed Computing}, PODC '24, 2024.
\newblock \doi{10.1145/3662158.3662802}.

\bibitem[Feng et~al.(2015)Feng, Xu, and Mannor]{feng2015distributedrobustlearning}
J.~Feng, H.~Xu, and S.~Mannor.
\newblock Distributed robust learning, 2015.

\bibitem[Fischer and Lynch(1982)]{FischerLynchMinRounds}
M.~J. Fischer and N.~A. Lynch.
\newblock {A Lower Bound for the Time to Assure Interactive Consistency}.
\newblock \emph{{Information Processing Letters}}, 14\penalty0 (4):\penalty0 183 -- 186, 1982.

\bibitem[Ghinea et~al.(2023)Ghinea, Liu-Zhang, and Wattenhofer]{10.1145/3558481.3591105}
D.~Ghinea, C.-D. Liu-Zhang, and R.~Wattenhofer.
\newblock Multidimensional approximate agreement with asynchronous fallback.
\newblock In \emph{Proceedings of the 35th ACM Symposium on Parallelism in Algorithms and Architectures}, SPAA '23, 2023.
\newblock \doi{10.1145/3558481.3591105}.

\bibitem[Grivet~Sébert et~al.(2021)Grivet~Sébert, Pinot, Zuber, Gouy-Pailler, and Sirdey]{Grivet_S_bert_2021}
A.~Grivet~Sébert, R.~Pinot, M.~Zuber, C.~Gouy-Pailler, and R.~Sirdey.
\newblock Speed: secure, private, and efficient deep learning.
\newblock \emph{Machine Learning}, 110\penalty0 (4):\penalty0 675–694, Mar. 2021.
\newblock \doi{10.1007/s10994-021-05970-3}.

\bibitem[Guerraoui et~al.(2021)Guerraoui, Guirguis, Plassmann, Ragot, and Rouault]{guerraoui2021garfield}
R.~Guerraoui, A.~Guirguis, J.~Plassmann, A.~Ragot, and S.~Rouault.
\newblock Garfield: System support for byzantine machine learning (regular paper).
\newblock In \emph{2021 51st Annual IEEE/IFIP International Conference on Dependable Systems and Networks (DSN)}, 2021.
\newblock \doi{10.1109/DSN48987.2021.00021}.

\bibitem[Guerraoui et~al.(2024)Guerraoui, Gupta, and Pinot]{byzantine-primer}
R.~Guerraoui, N.~Gupta, and R.~Pinot.
\newblock Byzantine machine learning: A primer.
\newblock \emph{ACM Comput. Surv.}, 56\penalty0 (7), Apr. 2024.
\newblock \doi{10.1145/3616537}.

\bibitem[Jere et~al.(2021)Jere, Farnan, and Koushanfar]{9308910}
M.~S. Jere, T.~Farnan, and F.~Koushanfar.
\newblock A taxonomy of attacks on federated learning.
\newblock \emph{IEEE Security \& Privacy}, 19\penalty0 (2):\penalty0 20--28, 2021.
\newblock \doi{10.1109/MSEC.2020.3039941}.

\bibitem[Jin et~al.(2020)Jin, Huang, He, Dai, and Wu]{jin2020signSGD}
R.~Jin, Y.~Huang, X.~He, H.~Dai, and T.~Wu.
\newblock Stochastic-sign sgd for federated learning with theoretical guarantees.
\newblock \emph{arXiv preprint arXiv:2002.10940}, 2020.

\bibitem[Kairouz et~al.(2021)Kairouz, McMahan, Avent, Bellet, Bennis, Bhagoji, Bonawit, Charles, Cormode, Cummings, D’Oliveira, Eichner, El~Rouayheb, Evans, Gardner, Garrett, Gascón, Ghazi, Gibbons, Gruteser, Harchaoui, He, He, Huo, Hutchinson, Hsu, Jaggi, Javidi, Joshi, Khodak, Konecný, Korolova, Koushanfar, Koyejo, Lepoint, Liu, Mittal, Mohri, Nock, Özgür, Pagh, Qi, Ramage, Raskar, Raykova, Song, Song, Stich, Sun, Theertha~Suresh, Tramèr, Vepakomma, Wang, Xiong, Xu, Yang, Yu, Yu, and Zhao]{9464278}
P.~Kairouz, H.~B. McMahan, B.~Avent, A.~Bellet, M.~Bennis, A.~N. Bhagoji, K.~Bonawit, Z.~Charles, G.~Cormode, R.~Cummings, R.~G.~L. D’Oliveira, H.~Eichner, S.~El~Rouayheb, D.~Evans, J.~Gardner, Z.~Garrett, A.~Gascón, B.~Ghazi, P.~B. Gibbons, M.~Gruteser, Z.~Harchaoui, C.~He, L.~He, Z.~Huo, B.~Hutchinson, J.~Hsu, M.~Jaggi, T.~Javidi, G.~Joshi, M.~Khodak, J.~Konecný, A.~Korolova, F.~Koushanfar, S.~Koyejo, T.~Lepoint, Y.~Liu, P.~Mittal, M.~Mohri, R.~Nock, A.~Özgür, R.~Pagh, H.~Qi, D.~Ramage, R.~Raskar, M.~Raykova, D.~Song, W.~Song, S.~U. Stich, Z.~Sun, A.~Theertha~Suresh, F.~Tramèr, P.~Vepakomma, J.~Wang, L.~Xiong, Z.~Xu, Q.~Yang, F.~X. Yu, H.~Yu, and S.~Zhao.
\newblock \emph{Advances and Open Problems in Federated Learning}.
\newblock 2021.

\bibitem[Karimireddy et~al.(2021)Karimireddy, He, and Jaggi]{pmlr-v139-karimireddy21a}
S.~P. Karimireddy, L.~He, and M.~Jaggi.
\newblock Learning from history for byzantine robust optimization.
\newblock In \emph{Proceedings of the 38th International Conference on Machine Learning}, volume 139 of \emph{Proceedings of Machine Learning Research}, pages 5311--5319. PMLR, 18--24 Jul 2021.
\newblock URL \url{https://proceedings.mlr.press/v139/karimireddy21a.html}.

\bibitem[Konečný et~al.(2015)Konečný, McMahan, and Ramage]{konečný2015federated}
J.~Konečný, B.~McMahan, and D.~Ramage.
\newblock Federated optimization:distributed optimization beyond the datacenter, 2015.

\bibitem[Konečný et~al.(2017)Konečný, McMahan, Yu, Richtárik, Suresh, and Bacon]{konečný2017federated}
J.~Konečný, H.~B. McMahan, F.~X. Yu, P.~Richtárik, A.~T. Suresh, and D.~Bacon.
\newblock Federated learning: Strategies for improving communication efficiency, 2017.

\bibitem[Kuhn(1973)]{kuhn-geom-history}
H.~W. Kuhn.
\newblock A note on fermat's problem.
\newblock \emph{Mathematical programming}, 4:\penalty0 98--107, 1973.

\bibitem[Lamport et~al.(1982)Lamport, Shostak, and Pease]{10.1145/357172.357176}
L.~Lamport, R.~Shostak, and M.~Pease.
\newblock The byzantine generals problem.
\newblock \emph{ACM Trans. Program. Lang. Syst.}, 4\penalty0 (3):\penalty0 382–401, July 1982.
\newblock ISSN 0164-0925.
\newblock \doi{10.1145/357172.357176}.
\newblock URL \url{https://doi.org/10.1145/357172.357176}.

\bibitem[Lian et~al.(2017)Lian, Zhang, Zhang, Hsieh, Zhang, and Liu]{decentralized-outperform}
X.~Lian, C.~Zhang, H.~Zhang, C.-J. Hsieh, W.~Zhang, and J.~Liu.
\newblock Can decentralized algorithms outperform centralized algorithms? a case study for decentralized parallel stochastic gradient descent.
\newblock In \emph{Proceedings of the 31st International Conference on Neural Information Processing Systems}, NIPS'17, 2017.

\bibitem[McMahan et~al.(2017{\natexlab{a}})McMahan, Moore, Ramage, Hampson, and Arcas]{fedavg}
B.~McMahan, E.~Moore, D.~Ramage, S.~Hampson, and B.~A.~y. Arcas.
\newblock {Communication-Efficient Learning of Deep Networks from Decentralized Data}.
\newblock In A.~Singh and J.~Zhu, editors, \emph{Proceedings of the 20th International Conference on Artificial Intelligence and Statistics}, volume~54 of \emph{Proceedings of Machine Learning Research}, pages 1273--1282. PMLR, 20--22 Apr 2017{\natexlab{a}}.
\newblock URL \url{https://proceedings.mlr.press/v54/mcmahan17a.html}.

\bibitem[McMahan et~al.(2017{\natexlab{b}})McMahan, Moore, Ramage, Hampson, and Arcas]{pmlr-v54-mcmahan17a}
B.~McMahan, E.~Moore, D.~Ramage, S.~Hampson, and B.~A.~y. Arcas.
\newblock {Communication-Efficient Learning of Deep Networks from Decentralized Data}.
\newblock In \emph{Proceedings of the 20th International Conference on Artificial Intelligence and Statistics}, volume~54 of \emph{Proceedings of Machine Learning Research}, pages 1273--1282. PMLR, 20--22 Apr 2017{\natexlab{b}}.
\newblock URL \url{https://proceedings.mlr.press/v54/mcmahan17a.html}.

\bibitem[McMahan et~al.(2016)McMahan, Moore, Ramage, and y~Arcas]{google-federated}
H.~B. McMahan, E.~Moore, D.~Ramage, and B.~A. y~Arcas.
\newblock Federated learning of deep networks using model averaging.
\newblock \emph{arXiv preprint arXiv:1602.05629}, 2016.

\bibitem[Mendes and Herlihy(2013)]{10.1145/2488608.2488657}
H.~Mendes and M.~Herlihy.
\newblock Multidimensional approximate agreement in byzantine asynchronous systems.
\newblock In \emph{Proceedings of the Forty-Fifth Annual ACM Symposium on Theory of Computing}, STOC '13, 2013.
\newblock \doi{10.1145/2488608.2488657}.

\bibitem[Mendes et~al.(2015)Mendes, Herlihy, Vaidya, and Garg]{multidim-approx-agreement}
H.~Mendes, M.~Herlihy, N.~Vaidya, and V.~K. Garg.
\newblock Multidimensional agreement in byzantine systems.
\newblock \emph{Distrib. Comput.}, 28\penalty0 (6):\penalty0 423–441, dec 2015.
\newblock \doi{10.1007/s00446-014-0240-5}.

\bibitem[Noble et~al.(2022)Noble, Bellet, and Dieuleveut]{noble2023differentiallyprivatefederatedlearning}
M.~Noble, A.~Bellet, and A.~Dieuleveut.
\newblock Differentially private federated learning on heterogeneous data.
\newblock In \emph{Proceedings of The 25th International Conference on Artificial Intelligence and Statistics}, volume 151 of \emph{Proceedings of Machine Learning Research}, pages 10110--10145. PMLR, 28--30 Mar 2022.
\newblock URL \url{https://proceedings.mlr.press/v151/noble22a.html}.

\bibitem[Park and Lee(2024)]{park2024signsgd}
C.~Park and N.~Lee.
\newblock {S}ign{SGD} with federated defense: Harnessing adversarial attacks through gradient sign decoding.
\newblock In \emph{Proceedings of the 41st International Conference on Machine Learning}, volume 235 of \emph{Proceedings of Machine Learning Research}, pages 39762--39780. PMLR, 21--27 Jul 2024.
\newblock URL \url{https://proceedings.mlr.press/v235/park24h.html}.

\bibitem[Pillutla et~al.(2022)Pillutla, Kakade, and Harchaoui]{pillutla2022robust}
K.~Pillutla, S.~M. Kakade, and Z.~Harchaoui.
\newblock Robust aggregation for federated learning.
\newblock \emph{IEEE Transactions on Signal Processing}, 70:\penalty0 1142--1154, 2022.
\newblock \doi{10.1109/TSP.2022.3153135}.

\bibitem[Shi et~al.(2022)Shi, Wan, Hu, Lu, and Yu~Zhang]{shi2022challenges}
J.~Shi, W.~Wan, S.~Hu, J.~Lu, and L.~Yu~Zhang.
\newblock { Challenges and Approaches for Mitigating Byzantine Attacks in Federated Learning }.
\newblock In \emph{2022 IEEE International Conference on Trust, Security and Privacy in Computing and Communications (TrustCom)}, Dec. 2022.
\newblock \doi{10.1109/TrustCom56396.2022.00030}.

\bibitem[Small(1990)]{geom-definition}
C.~G. Small.
\newblock A survey of multidimensional medians.
\newblock \emph{International Statistical Review}, 58:\penalty0 263--277, 1990.
\newblock URL \url{https://api.semanticscholar.org/CorpusID:121808100}.

\bibitem[Srikanth and Toueg(1987)]{TouegRB}
T.~Srikanth and S.~Toueg.
\newblock {Simulating Authenticated Broadcasts to Derive Simple Fault-Tolerant Algorithms}.
\newblock \emph{{Distributed Computing}}, 2\penalty0 (2):\penalty0 80--94, June 1987.

\bibitem[Su and Vaidya(2016)]{Su-Vaidya-2016}
L.~Su and N.~H. Vaidya.
\newblock Non-bayesian learning in the presence of byzantine agents.
\newblock In C.~Gavoille and D.~Ilcinkas, editors, \emph{Distributed Computing}, pages 414--427, Berlin, Heidelberg, 2016. Springer Berlin Heidelberg.
\newblock ISBN 978-3-662-53426-7.

\bibitem[Vaidya and Garg(2013)]{10.1145/2484239.2484256}
N.~H. Vaidya and V.~K. Garg.
\newblock Byzantine vector consensus in complete graphs.
\newblock In \emph{Proceedings of the 2013 ACM Symposium on Principles of Distributed Computing}, PODC '13, 2013.
\newblock \doi{10.1145/2484239.2484256}.

\bibitem[Wei et~al.(2020)Wei, Li, Ding, Ma, Yang, Farokhi, Jin, Quek, and Vincent~Poor]{differential-privacy-fed-learning}
K.~Wei, J.~Li, M.~Ding, C.~Ma, H.~H. Yang, F.~Farokhi, S.~Jin, T.~Q.~S. Quek, and H.~Vincent~Poor.
\newblock Federated learning with differential privacy: Algorithms and performance analysis.
\newblock \emph{IEEE Transactions on Information Forensics and Security}, 15:\penalty0 3454--3469, 2020.
\newblock \doi{10.1109/TIFS.2020.2988575}.

\bibitem[Weiszfeld(1937)]{weiszfeld1937point}
E.~Weiszfeld.
\newblock Sur le point pour lequel la somme des distances de n points donn{\'e}s est minimum.
\newblock \emph{Tohoku Mathematical Journal, First Series}, 43:\penalty0 355--386, 1937.

\bibitem[Weiszfeld and Plastria(2009)]{weiszfeld2009point}
E.~Weiszfeld and F.~Plastria.
\newblock On the point for which the sum of the distances to n given points is minimum.
\newblock \emph{Annals of Operations Research}, 167:\penalty0 7--41, 2009.

\bibitem[Yin et~al.(2018)Yin, Chen, Kannan, and Bartlett]{cwmDEF}
D.~Yin, Y.~Chen, R.~Kannan, and P.~Bartlett.
\newblock {B}yzantine-robust distributed learning: Towards optimal statistical rates.
\newblock In \emph{Proceedings of the 35th International Conference on Machine Learning}, volume~80 of \emph{Proceedings of Machine Learning Research}, pages 5650--5659. PMLR, 10--15 Jul 2018.
\newblock URL \url{https://proceedings.mlr.press/v80/yin18a.html}.

\bibitem[Zhang et~al.(2020)Zhang, Li, Xia, Wang, Yan, and Liu]{zhang2020batchcrypt}
C.~Zhang, S.~Li, J.~Xia, W.~Wang, F.~Yan, and Y.~Liu.
\newblock $\{$BatchCrypt$\}$: Efficient homomorphic encryption for $\{$Cross-Silo$\}$ federated learning.
\newblock In \emph{2020 USENIX annual technical conference (USENIX ATC 20)}, pages 493--506, 2020.
\newblock URL \url{https://www.usenix.org/conference/atc20/presentation/zhang-chengliang}.

\bibitem[Zhao et~al.(2018)Zhao, Li, Lai, Suda, Civin, and Chandra]{zhao2018federated}
Y.~Zhao, M.~Li, L.~Lai, N.~Suda, D.~Civin, and V.~Chandra.
\newblock Federated learning with non-iid data.
\newblock 2018.
\newblock \doi{10.48550/ARXIV.1806.00582}.

\end{thebibliography}

\end{document}